\documentclass[conference]{IEEEtran}
\IEEEoverridecommandlockouts
\usepackage{cite}
\usepackage{mathtools}
\usepackage{amsmath}
\usepackage{amsthm}

\usepackage{amsmath,amssymb,amsfonts}
\usepackage{algorithm}
\usepackage{algorithmic}

\usepackage{graphicx}
\usepackage{textcomp}
\usepackage{xcolor}

\newtheorem{thm}{\protect\theoremname}
  \theoremstyle{plain}
  \newtheorem{lem}[thm]{Lemma}
  \theoremstyle{plain}
  \newtheorem{prop}[thm]{Proposition}
  
\def\BibTeX{{\rm B\kern-.05em{\sc i\kern-.025em b}\kern-.08em
    T\kern-.1667em\lower.7ex\hbox{E}\kern-.125emX}}
\begin{document}

\title{Decision-Oriented Communications: Application to Energy-Efficient Resource Allocation}
\author{\IEEEauthorblockN{Hang Zou\IEEEauthorrefmark{1},
Chao Zhang\IEEEauthorrefmark{1}, Samson Lasaulce\IEEEauthorrefmark{1}, Lucas Saludjian\IEEEauthorrefmark{2}, and Patrick Panciatici\IEEEauthorrefmark{2}}

\IEEEauthorblockA{\IEEEauthorrefmark{1}LSS, CNRS-CentraleSupelec-Univ. Paris Sud, Gif-sur-Yvette, France}
\IEEEauthorblockA{\IEEEauthorrefmark{2}RTE, France}}

\maketitle

\begin{abstract}
In this paper, we introduce the problem of decision-oriented communications, that is, the goal of the source is to send the right amount of information in order for the intended destination to execute a task. More specifically, we restrict our attention to how the source should quantize information so that the destination can maximize a utility function which represents the task to be executed  only  knowing the quantized information. For example, for utility functions under the form $u\left(\boldsymbol{x};\ \boldsymbol{g}\right)$, $\boldsymbol{x}$ might represent a decision in terms of using some radio resources and $\boldsymbol{g}$ the system state  which is only observed through its quantized version $Q(\boldsymbol{g})$. Both in the case where the utility function is known and the case where it is only observed through its realizations, we provide solutions to determine such a quantizer. We show how this approach applies to energy-efficient power allocation. In particular, it is seen that quantizing the state very roughly is perfectly suited to sum-rate-type function maximization, whereas energy-efficiency metrics are more sensitive to imperfections. 

%
%
%
%
%
%
%
\end{abstract}

\begin{IEEEkeywords}
Quantization, Resource Allocation, Learning, Neural Networks, Energy-Efficiency.
\end{IEEEkeywords}

\section{Introduction}
\label{sec:Introdution}
Since the pioneering and fundamental works of Shannon, the ultra dominant paradigm for designing a communication system is that communications must satisfy quality requirements. Typically, the bit error rate, the packet error rate, the outage probability or the distortion must be minimized. It turns out that the classical paradigm consisting in pursuing communication reliability or possibly security may be not suited to scenarios such as systems where communications occur in order for a given task to be executed. To be more concrete, in many resource allocation problems, some knowledge about the system state is necessary to make a decision in terms of using the available resources but having an accurate or too reliable knowledge about the state might induce a prohibitive amount of signaling, while only leading to a marginal increase of the system performance, say in terms of utility function. This motivates us to develop a new communication paradigm, which we refer to as decision-oriented communications (DOC). In the scope of this paper, more modestly, we restrict our attention to one given operation of the design of decision-oriented transmitters, that is decision-oriented quantization (DOQ). 

In this paper, we assume that the task the system should execute can be represented as an optimization problem. We assume that the ultimate objective to be reached is to maximize a function $u\left(\boldsymbol{x};\ \boldsymbol{g}\right)$ (called the utility function) with respect to the decision variables $\boldsymbol{x}$ while having an imperfect knowledge of the function parameters $\boldsymbol{g}$. Technically, we want to know how to design a device which has to quantize the actual parameters $\boldsymbol{g}$ before sending them to the decision-making entity which has to choose the decision variables $\boldsymbol{x}$. A classical instance of such a problem occurs when a receiver has to feedback information (e.g., channel state information CSI) to a transmitter so that the latter adapts its transmission scheme; this is the problem of quantized CSI.

When inspecting the literature of communications, it appears that the problem under consideration has not been explored yet, and definitely not from the concrete perspective approach taken in this paper. There exist some works on goal-oriented communications such as \cite{goldreich-jacm-2012} but those rely on logical aspects of computer science and do not formulate the problem as an optimization problem and do not tackle the problem from the perspective of coding and thus quantization. 
distortion. Concerning the problem of quantization, there exist some works (e.g., \cite{wang-ssp-2012}) where the performance criterion is not distortion (as originally proposed in \cite{lloyd-TIT-1982}) but a more general Lp-norm. But the approach developed in the present paper is not only more general because it concerns arbitrary utility functions but also corresponds to a new approach of designing a communication system. In particular, both source coding and channel coding should be revisited. Additionally, we tackle the important case where the performance criterion (i.e., the utility function) is not known but only observed through its realizations, which is different from the case studied in \cite{Zhang-wiopt-2016}.

%

The paper is structured as follows. In Sec.~\ref{sec:Problem Formulation} we introduce the problem of decision-oriented quantization. In Sec.~\ref{sec:General Algorithm for GOQ}, we develop two solutions to determine such a quantizer, the solutions respectively corresponding to the case where the utility function is known and the case where only its realizations can be observed. In Sec.~\ref{sec:Application To Energy Efficiency Communication}, we show how to apply our approach to the problem of energy-efficient resource allocation for a multiband and MIMO communications. The paper is concluded by Sec.~\ref{sec:Conclusion}.

\section{Problem Formulation}
\label{sec:Problem Formulation}

Consider the following utility function $u\left(\boldsymbol{x};\ \boldsymbol{g}\right)$,
where $\boldsymbol{x}$ represents the resources that transmitter
can allocate and $\boldsymbol{g}$ is an unknown environment parameter
that is only measurable but uncontrollable. The utility function can be for example
, the energy efficiency, transmission rate and transmission error rate.
Let $\mathcal{X}$ be the decision space which is generally $\mathbb{R}^{N}$
with $N$ the number of users and $\mathcal{G}$ be the parameter space which can
be $\mathbb{R}^{K}$ or $\mathbb{C}^{K}$ depending on the situation
with $K$ the dimension of vector $\boldsymbol{g}$.
Obviously, given a fixed $\boldsymbol{g}$, one is capable to maximize the utility function:
\begin{equation}
\label{theoretical maximum}
U\left(\boldsymbol{g}\right)=\max_{\boldsymbol{x}\in\mathfrak{\mathcal{X}}}u\left(\boldsymbol{x};\ \boldsymbol{g}\right)
\end{equation}
and find maximum point of this function:
\begin{equation}
\label{theoretical maximizer}
\boldsymbol{x}^{*}\left(\boldsymbol{g}\right)\in\arg\max_{\boldsymbol{x}\in\mathfrak{\mathcal{X}}}u\left(\boldsymbol{x};\ \boldsymbol{g}\right)
\end{equation}
However, knowing  all information about the environment e.g., the channel state
information may be too demanding in practical application or even consumptive.
Providing only with an estimated version (or even a quantized version ) of CSI
may not change the decision of the transmitter without losing the optimality. A quantizer is just a function s.t. $\mathcal{Q}\left(\boldsymbol{g}\right)=\widehat{\boldsymbol{g}}$
. More precisely, a quantizer divides the entire space $\mathcal{G}$
of the parameter $\boldsymbol{g}$ into several cells $\mathcal{C}_{1},\dots,\mathcal{C}_{M}$
s.t. 
\begin{align*}
\forall i\neq j,\ \mathcal{C}_{i}\bigcap\mathcal{C}_{j} & =\varnothing\ \textrm{and}\ \bigcup_{i=1}^{M}\mathcal{C}_{i}=\mathcal{G}
\end{align*}
with some corresponding representatives $\boldsymbol{r}_{1},\dots,\boldsymbol{r}_{M}$.
The quantization rule is :
\begin{equation}
\label{eq:quantization rule of classic quantizer}
\mathcal{Q}\left(\boldsymbol{g}\right)=\boldsymbol{r}_{k},\ \textrm{if}\ \boldsymbol{g}\in\mathcal{C}_{k}
\end{equation}
In this scenario, the optimization procedure can still be performed
using a quantized version of the parameter $\widehat{\boldsymbol{g}}$:
\begin{equation}
\label{eq: intuitive achievable maximum by quantization}
\widehat{U}\left(\mathcal{Q}\left(\boldsymbol{g}\right)\right)=\max_{\boldsymbol{x}\in\mathfrak{\mathcal{X}}}u\left(\boldsymbol{x};\ \boldsymbol{\widehat{g}}\right)=\max_{\boldsymbol{x}\in\mathfrak{\mathcal{X}}}u\left(\boldsymbol{x};\ \mathcal{Q}\left(\boldsymbol{g}\right)\right) 
\end{equation}
One can find the maximum point (saying the quantized maximum point ) for the
the utility function using a quantized version of $\boldsymbol{g}$:
\begin{equation}
\label{eq:quantized optimizer}
\boldsymbol{\widehat{x}}^{*}\left(\boldsymbol{\widehat{g}}\right)\in\arg\max_{\boldsymbol{x}\in\mathfrak{\mathcal{X}}}u\left(\boldsymbol{x};\ \boldsymbol{\widehat{g}}\right)
\end{equation}



Even if the maximum point is found based on a quantized version of the parameter, the utility
function still undergoes a realistic channel. 
Therefore the realistic maximum of the utility function knowing merely
a quantized channel parameter can be written as:
\begin{equation}
\label{eq:final goal of GOQ}
\widetilde{U}\left(\mathcal{Q}\left(\boldsymbol{g}\right)\right)=u\left(\boldsymbol{\widehat{x}}^{*}\left(\boldsymbol{\widehat{g}}\right);\ \boldsymbol{g}\right)
\end{equation}

The objective of DOQ is to find a quantizer $\mathcal{Q^{\mathrm{*}}}$
such that:

\begin{equation}
\mathcal{Q^{\mathrm{*}}}\in\arg\min_{\mathcal{Q}}\mathbb{E}_{\boldsymbol{g}}\left[\left\Vert U\left(\boldsymbol{g}\right)-\widetilde{U}\left(\mathcal{Q}\left(\boldsymbol{g}\right)\right)\right\Vert ^{2}\right]\label{eq: goal of GOQ}
\end{equation}

The expectation is taken over the probability density distribution (p.d.f) $\phi\left(\boldsymbol{g}\right)$ of the parameter $\boldsymbol{g}$.

\section{Proposed solutions for arbitrary utility functions}
\label{sec:General Algorithm for GOQ}
As  explained in Section \ref{sec:Problem Formulation}, the
goal of the DOQ is to find the optimal quantizer that minimizes the
optimality loss (in terms of maximizing the utility function) which is induced by imperfect knowledge of $\boldsymbol{g}$. In general, this problem
can be too complicated to solve due to the continuity of the decision
space $\mathcal{X}$. Assume that the effective decision set is finite due to
the fact the feedback information from the receiver is limited:

\begin{equation}
\label{eq:discrete decision space assumption}
\mathcal{D}=\left\{ \boldsymbol{d}_{1},\dots,\boldsymbol{d}_{M}\right\} \subseteq\mathcal{X},\ M<+\infty
\end{equation}

This happens for examples in many power control systems in some cellular
communication standards or even an equivalent optimality between $\mathcal{D}$
and $\mathcal{X}$ proved in \cite{bcpsr2008}. In what follows, only
discrete decision space will be considered. Moreover, for the sake of simplicity, we only
use vector notation of the decision. The extension to matrix-form decision can be treated in the same way and will be presented in Sec. IV-B.

Due to this assumption, the desired quantizer is equivalent to find
a pair $\left(\mathcal{D},\mathcal{C}\right)$ that maximizes the
expectation of our utility function, where  $\mathcal{C}=\left\{ \mathcal{C}_{1},\dots,\mathcal{C}_{M}\right\} $
is the quantization region set. A \emph{decisional quantizer} is just
a mapping that maps a given parameter $\boldsymbol{g}$ to a unique
decision $\boldsymbol{x}$:
\begin{align}
f\ :\ \left\{ \mathcal{C}_{1},\dots,\mathcal{C}_{M}\right\}  & \rightarrow\mathcal{D}\nonumber \\
\boldsymbol{g} & \mapsto\boldsymbol{d}\label{eq:definition of a decisional quantizer function}
\end{align}
Therefore a decisional quantizer \textbf{\textcolor{black}{\emph{
yields not the quantization value of the parameter but instead a unique
decision associated with it}}}. Obviously the optimal pair $\left(\mathcal{D}^{*},\mathcal{C}^{*}\right)$
is given by:

\begin{equation}
\left(\mathcal{D}^{*},\mathcal{C}^{*}\right)\in\arg\max_{\left(\mathcal{D},\mathcal{C}\right)}\mathbb{E}_{\boldsymbol{g}}\left[u\left(\boldsymbol{x};\ \boldsymbol{g}\right)\left|\left(\mathcal{D},\mathcal{C}\right)\right.\right]\label{eq:opitmal goal-oriented quantizer pair}
\end{equation}

where 
\begin{equation}
\mathbb{E}_{\boldsymbol{g}}\left[u\left(\boldsymbol{x};\ \boldsymbol{g}\right)\left|\left(\mathcal{D},\mathcal{C}\right)\right.\right]=\sum_{k=1}^{M}\int_{\mathcal{C}_{k}}u\left(\boldsymbol{d}_{k};\ \boldsymbol{g}\right)\phi\left(\boldsymbol{g}\right)d\boldsymbol{g}\label{eq:utility expectation in region-sum}
\end{equation}

\subsection{Model-based solution (known utility function)}
\label{subsec:Model-based Scenario}
We firstly assume that the utility function is known, i.e., an explicit expression of $u\left(\boldsymbol{x};\ \boldsymbol{g}\right)$ is available.

However, finding $\left(\mathcal{D}^{*},\mathcal{C}^{*}\right)$ jointly
is complicated, the DOQ problem can be split in two steps as the
classical quantization algorithm does:
\begin{enumerate}
\item The representative-to-cell step which is essentially 
finding the optimal quantization region (cells) given
the concrete decision space $\mathcal{D}$ (representatives):
\begin{equation}
\mathcal{C}_{k}^{*}=\left\{ \boldsymbol{g}\in\mathcal{G}\left|u\left(\boldsymbol{d}_{k};\ \boldsymbol{g}\right)=\max_{l}u\left(\boldsymbol{d}_{l};\ \boldsymbol{g}\right)\right.\right\}
\label{eq:initial definition for decisional quantizer}
\end{equation}

where $\mathcal{C}_{k}^{*}$ is called the  decision region
corresponding to decision $\boldsymbol{d}_{k}$.
\item The cell-to-representative step which 
is essentially finding the optimal decision space
(representatives) given the concrete quantization regions (cells) $\mathcal{C}$:
\begin{equation}
\boldsymbol{d}_{k}^{*}\in\arg\max_{\boldsymbol{d}\in\mathcal{X}}\int_{\mathcal{C}_{k}}u\left(\boldsymbol{d};\ \boldsymbol{g}\right)\phi\left(\boldsymbol{g}\right)d\boldsymbol{g}
\label{eq:optimal discrete decision space}
\end{equation}
\end{enumerate}
Furthermore, if step $1$ and step $2$ can be separately solved, designing an  algorithm which performs step $1$ and step $2$ in an iterative manner to find the
optimal quantizer will be possible, as the  classical
quantization algorithm, e.g., Lloyd-Max Algorithm \cite{lloyd-TIT-1982}  operates. The algorithm is summarized in Algo.~\ref{alg:Decisional-Quantization-Algorithm}. The convergence of Algo.~\ref{alg:Decisional-Quantization-Algorithm} can be proved by induction and is omitted here.
However, quiet often, an explicit expression of the utility function is difficult to obtain or maybe never exists. For example, the global performance of a massive cellular network may be too complicated to know an explicit expression. Nevertheless, we are always able to obtain the value of a specific realization.
Many Machine Learning tools can  kick in in this scenario effortlessly without knowing \emph{a priori} information about the utility function.

\begin{algorithm}[h]
\begin{algorithmic}[1] 
\STATE\ \ Input: utility function $u\left(\boldsymbol{x};\ \boldsymbol{g}\right)$

\STATE\ \ Input: error tolerance $\epsilon$ and max iteration
$T$

\STATE\ \ Input: initial decision set $\mathcal{D}^{\left(0\right)}=\left\{ \boldsymbol{d}_{1}^{\left(0\right)},\dots,\boldsymbol{d}_{M}^{\left(0\right)}\right\} $

\STATE\ \ Input: initial decision region $\mathcal{C}^{\left(0\right)}=\left\{ \mathcal{C}_{1}^{\left(0\right)},\dots,\mathcal{C}_{M}^{\left(0\right)}\right\} $

\STATE\ \ Output: optimal decision pair $\left(\mathcal{D}^{*},\mathcal{C}^{*}\right)$

\STATE\ \ Initialization: set iteration index $i\rightarrow1$

\STATE\ \ $\mathbf{do}$

\STATE\ \ \ \ \ Update iteration index $i\leftarrow i+1$

\STATE\ \textcolor{white}{ass}For all $k\in\left\{ 1,\dots,M\right\} $, update
$\mathcal{C}_{k}^{\left(i\right)}$ from $\mathcal{D}^{\left(i-1\right)}$
using (\ref{eq:initial definition for decisional quantizer})

\STATE\ \ \ \ \ For all $k\in\left\{ 1,\dots,M\right\} $, update
$\boldsymbol{d}_{k}^{\left(i\right)}$ from $\mathcal{C}_{k}^{\left(i\right)}$
using (\ref{eq:optimal discrete decision space})

\STATE\ \ $\boldsymbol{\mathbf{while}}$ $\sum_{k=1}^{M}\left\Vert \boldsymbol{d}_{k}^{\left(i\right)}-\boldsymbol{d}_{k}^{\left(i-1\right)}\right\Vert ^{2}>\varepsilon$
and $i\leq T$

\STATE\ \ $\left(\mathcal{D}^{*},\mathcal{C}^{*}\right)\leftarrow\left(\mathcal{D}^{\left(T\right)},\mathcal{C}^{\left(T\right)}\right)$

\end{algorithmic}

\caption{Decisional Quantization Algorithm \label{alg:Decisional-Quantization-Algorithm}}
\end{algorithm}

\subsection{Model-free approach (unknown utility function)}
\label{subsec:Model-free Scenario}

As we have explained, providing a systematic analytical procedure to partition the parameter space without knowing the explicit expression of the utility function can be very difficult. In this paper, we propose to solve this problem by using a Feed-forward Neural Network (FNN) based learning procedure. Denoting $W_{i,j}^{\left(l\right)}$ the weight between the neuron $i$ in
the $l$-th layer and the the neuron $j$ in $\left(l+1\right)$-th
layer and $b_{j}^{\left(l\right)}$ the bias term for neuron $j$, the basic model for FNN is given
by:

\begin{equation}
o_{j}^{\left(l+1\right)}=f\left(b_{j}^{\left(l\right)}+\sum_{i=1}^{Nd}W_{i,j}^{\left(l\right)}o_{i}^{\left(l\right)}\right)\label{eq: basic relation for a neuron}
\end{equation}

where $o_{j}^{\left(l+1\right)}$ is the output of the neuron $j$
and $o_{i}^{\left(l\right)}$ is the output of the neuron $i$ or
the input signal from neuron $i$ to neuron $j$ and $Nd$ is the number of neurons in each layer. $f\left(\cdot\right)$
is the activation function. Moreover we define the training set as $\mathcal{T}_{n}\coloneqq\left\{ \boldsymbol{g}_{t},\ \theta^{*}_{t}\right\} _{t=1}^{n}$, 
where $\theta_{t}^{*}$ is the optimal decision label  corresponding to
 parameter realization $\boldsymbol{g}_{t}$ gathered somehow,
e.g., the data already collected in the environment, analytical solution given the explicit expression of our utility function or even 
by exhaustive comparison between all possibles decisions:
\begin{equation}
\theta_{t}^{*}\in \arg\max_{\theta\in\left\{ 1,\dots,M\right\} }u\left(\boldsymbol{d}_{\theta};\boldsymbol{g}_{t}\right)
\end{equation}
\begin{figure}[htbp]
\begin{centering}
\includegraphics[scale=0.30]{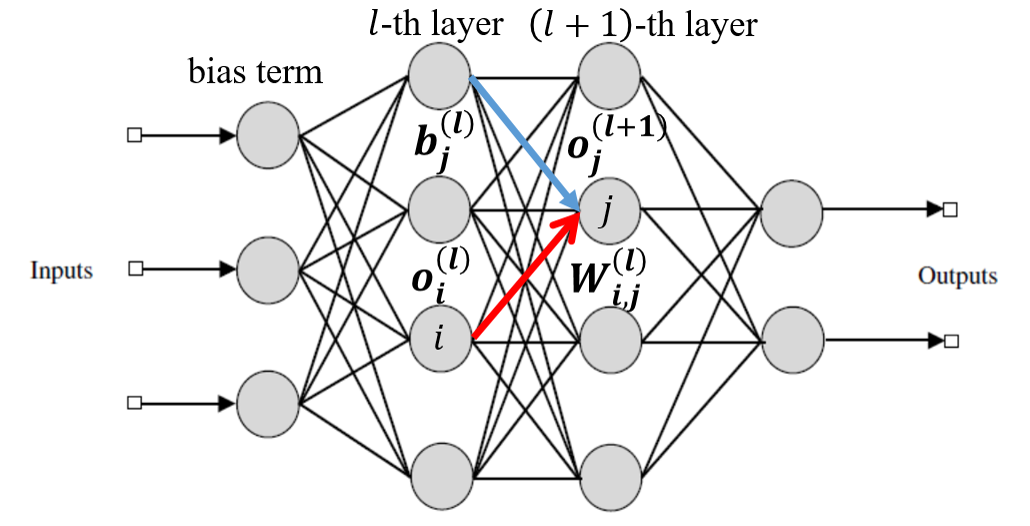}
\par\end{centering}
\caption{Basic structure of a FNN.}
\label{fig:basic FNN model}
\end{figure}

If the error estimation error (test error) is less than some threshold, the NN trained by this training set can give us a reasonable approximation of the real decisional quantizer. Again, if step 2 can be solved, we can use Alg. \ref{alg:Decisional-Quantization-Algorithm} to
solve the DOQ problem by replacing the original step 1 with our Machine Learning based approach.
This procedure can, in fact, be used for any utility function of the form $u(\textbf{x};\textbf{g})$.



\section{Application to Energy-Efficient Communications}
\label{sec:Application To Energy Efficiency Communication}
In this section, we consider a particular utility function, the Energy Efficiency (EE) which 
characterizes the efficiency of transmission in term of energy for multi-bands scenario and MIMO scenario. The general form of EE is given by:

\begin{equation}
u^{\textrm{EE}}\left(\boldsymbol{x};\boldsymbol{g}\right)\coloneqq\frac{F\left(\boldsymbol{x};\boldsymbol{g}\right)}{E\left(\boldsymbol{x}\right)}\label{eq:general form of EE}
\end{equation}

where $F\left(\boldsymbol{x};\boldsymbol{g}\right)$ represents the
lossless utility for the pair $\left(\boldsymbol{x};\boldsymbol{g}\right)$
and $E\left(\boldsymbol{x}\right)$ is the energy (or power) consumed
using policy $\boldsymbol{x}$.

\subsection{Multi-band Energy Efficiency}\label{subsec:Multi-band Energy Efficiency}
We consider the following EE function:

\begin{equation}
u^{\mathrm{MB}}\left(\boldsymbol{p};\boldsymbol{g}\right)=\frac{\sum_{i=1}^{N}w\left(\textrm{SNR}_{i}\right)}{\sum_{i=1}^{N}p_{i}}\label{eq:energy efficiency definition}
\end{equation}

where $i\in\left\{ 1,2,\dots,N\right\} $ is an index which might
represent the band, channel, or user index; $g_{i}>0$ is the channel
gain of $i$-th channel, $\boldsymbol{p}=\left(p_{1},\dots,p_{N}\right)$
is the power allocation vector; $\boldsymbol{g}=\left(g_{1},\dots,g_{N}\right)$
is the vector channels used by transmitter $i$; $\textrm{SNR}_{i}$
is the signal-to-noise ratio associated with channel $i$ chosen as
$\textrm{SNR}_{i}=\frac{p_{i}g_{i}}{\sigma^{2}}$ which suggests no
interference appears between bands, where $\sigma^{2}$ is the received
noise variance. Furthermore, the payoff function which represents
the packet success rate is chosen (see\cite{eepmat2011}).

\begin{equation}
w\left(s\right)=\exp\left(-\frac{c}{s}\right)\label{eq:packet success rate}
\end{equation}

where $c\geq0$ is a parameter related to spectral efficiency.

We firstly consider the single band scenario, i.e., $N=1$, we thus
have:

\begin{equation}
u\left(p;g\right)=\frac{\exp\left(-\frac{c\sigma^{2}}{pg}\right)}{p}\label{eq:energy-efficiency function for 1 user}
\end{equation}

Assume that $p\in\left\{ P_{1},\dots,P_{M}\right\} $ (without loss
of generality, we assume that $P_{1}<\dots<P_{M}$). We pick randomly
two decisions $P_{i}$ and $P_{j}$ (assume that $P_{i}<P_{j}$). 
On can obtain the following lemma \ref{lem:pairwise decision lemma}:
\begin{lem} 
\label{lem:pairwise decision lemma}
The optimal decisional threshold between two decisions $P_{i}$ and $P_{j}$ is:
\begin{equation}
g_{0}^{*}\left(P_{i},P_{j}\right)=\frac{c\sigma^{2}\left[\frac{1}{P_{i}}-\frac{1}{P_{j}}\right]}{\ln P_{j}-\ln P_{i}}>0,\ \textrm{if}\ P_{i}<P_{j}\label{eq:optimal thershold for scalar EE (M decisions)}
\end{equation}
\end{lem}

One may think there will be $\frac{M(M-1)}{2}$ possible thresholds
in this problem. However there are only $\left(M-1\right)$ effective
thresholds. A threshold is called effective if it is the boundary
of a decision region. We have the following proposition \ref{prop:effictive threshold between two decisions}: 
\begin{prop}
\label{prop:effictive threshold between two decisions}
Only the threshold in form of $g_{0}^{*}\left(P_{i},P_{i+1}\right),i=1,\dots,M-1$ is
effective.
\end{prop}

\begin{proof}
One can easily find that:

\begin{align}
\frac{\partial g_{0}^{*}\left(P_{i},P_{j}\right)}{\partial P_{i}} & =\frac{c\sigma^{2}}{P_{i}^{2}\left(\ln P_{j}-\ln P_{i}\right)^{2}}\left[\ln\left(\frac{P_{i}}{P_{j}}\right)-\frac{P_{i}}{P_{j}}+1\right]\nonumber \\
\frac{\partial g_{0}^{*}\left(P_{i},P_{j}\right)}{\partial P_{j}} & =\frac{c\sigma^{2}}{P_{j}^{2}\left(\ln P_{j}-\ln P_{i}\right)^{2}}\left[\ln\left(\frac{P_{j}}{P_{i}}\right)-\frac{P_{j}}{P_{i}}+1\right]
\label{eq:multi-level derivative of scalar EE}
\end{align}

One can prove that
\begin{equation}
f\left(x\right)=\ln x-x+1\geq0,\ \forall x>0   
\end{equation}

where equality is only taken when $x=1$. Knowing that $P_{i}<P_{j}$,
so we can conclude that

\begin{align}
\frac{\partial g_{0}^{*}\left(P_{i},P_{j}\right)}{\partial P_{i}} & <0\ \textrm{if}\ P_{i}<P_{j}\\
\frac{\partial g_{0}^{*}\left(P_{i},P_{j}\right)}{\partial P_{j}} & <0\ \textrm{if}\ P_{i}>P_{j}
\label{eq: sign of multi-level derivative of scalar EE}
\end{align}

Thus we have for any $i\neq j$ such that $P_{i}<P_{j}$ we always
have
\vspace{-0.1in}
\begin{align*}
g_{0}^{*}\left(P_{i},P_{i+1}\right)>\dots>g_{0}^{*}\left(P_{i},P_{j-1}\right)>g_{0}^{*}\left(P_{i},P_{j}\right)\\
g_{0}^{*}\left(P_{i},P_{j}\right)>g_{0}^{*}\left(P_{i+1},P_{j}\right)>\dots>g_{0}^{*}\left(P_{j-1},P_{j}\right)  
\end{align*}

So the conclusion is obvious.
\end{proof}

\subsection{Energy Efficiency in MIMO system}
\label{subsec:Energy Efficiency in MIMO system}
 We consider the following single user Multiple Input Multiple Output (MIMO) communication system. The receiving signal is modeled by:

\begin{equation}
\boldsymbol{y}=\mathbf{H}\boldsymbol{x}+\boldsymbol{z}\label{eq:MIMO system model}
\end{equation}

where $\mathbf{H}$ is the $N_{r}\times N_{t}$ channel transfer matrix
with $N_{t}$ transmit antennas and $N_{r}$ receive antennas. We
assume the entries of $\mathbf{H}$ are i.i.d. zero-mean circularly
symmetric complex Gaussian distributed according to $\mathcal{CN}\left(0,1\right)$.
A vector $\boldsymbol{x}$ is the transmitting symbols vector with
dimension $N_{t}$ and $\boldsymbol{z}$ is the receiving white Gaussian
noise vector distributed as $\mathcal{CN}\left(\boldsymbol{0},\sigma^{2}\boldsymbol{\mathbf{I}}_{N_{r}}\right)$.
Moreover $\boldsymbol{\mathbf{Q}}=\mathbb{E}\left[\boldsymbol{x}\boldsymbol{x}^{H}\right]$
denote the covariance matrix of $\boldsymbol{x}$ which determines
the power allocation policy. And we have the common maximum total
power constraint:

\begin{equation}
\textrm{Tr}\left(\boldsymbol{\mathbf{Q}}\right)\leq P_{max}\label{eq:power constraint in trace forme}
\end{equation}

Given this matrix-form of the system, we define the Energy
Efficiency as follows:

\begin{equation}
u^{\textrm{MIMO}}\left(\mathbf{Q};\mathbf{H}\right)=\frac{R_{0}\log_{2}\left|\boldsymbol{\mathbf{I}}_{N_{r}}+\rho\mathbf{HQ}\mathbf{H}^{H}\right|}{\textrm{Tr}\left(\boldsymbol{\mathbf{Q}}\right)+P_{0}}\label{eq:single user MIMO EE with P0}
\end{equation}
where  $\rho=\frac{1}{\sigma^{2}}$,  $R_{0}$  is the raw data rate (in bits/s) and  $P_{0}$  represents the power consumed by the transmitter when the radiated power is zero. For instance, in \cite{betz2008} it may represent the computation power or the circuit power.

The existence of $P_0$ is not only reasonable but also avoids
the following fact that the most efficient transmission occurs when
$p=\textrm{Tr}\left(\boldsymbol{\mathbf{Q}}\right)=0$. The global decision set is the Equal Gain Transmission (EGT) with
antenna selections. Without loss of generality, we only consider diagonal
covariance matrix of the transmission signal, i.e., $\boldsymbol{\mathbf{Q}}=\mathbf{Diag}\left(\boldsymbol{p}\right)$
with $\boldsymbol{p}=\left(p_{1},\dots,p_{N_{t}}\right)$. Where $\mathbf{Diag}\left(\boldsymbol{v}\right)$ generates the diagonal matrix  whose diagonal is exactly the vector $\boldsymbol{v}$. The decision
set is chosen as following form:

\begin{equation}
\mathcal{D}=\left\{ \boldsymbol{\mathbf{Q}}=\frac{P_{max}}{l}\mathbf{Diag}\left(\boldsymbol{e}\right)\left|\boldsymbol{e}\in\mathcal{S}_{l},\ \forall l\leq N_{t}\right.\right\} \label{eq:global decision set}
\end{equation}

where $\mathcal{S}_{l}=\left\{ \boldsymbol{e}\in\left\{ 0,1\right\} ^{N_{t}}\left|\sum_{i=1}^{N_{t}}\boldsymbol{e}_{i}=l\right.\right\} $
which is the set of $N_{t}$ -dimensional binary vector summing to
$l$. The decision set $\mathcal{D}_{k}$ associated to a decisional
quantizer with $k\leq2^{N_t}-1$ decisions can be constructed as follows iteratively:

\begin{align}
\mathcal{D}_{k} & =\begin{cases}
\left\{ \mathbf{Q}_{1}\right\}  & \mathbf{Q}_{1}\in\mathcal{D},\ k=1\\
\mathcal{D}_{k-1}\cup\left\{ \mathbf{Q}_{k}\right\}  & \mathbf{Q}_{k}\in\mathcal{D}\backslash\mathcal{D}_{k}
\end{cases}\label{eq:decision set iterative generation}
\end{align}

The singleton set is chosen among all possible sets randomly. Consider the optimality of the decision set, we choose the maximum
total power $P_{max}=P^{*}$ s.t.

\begin{equation}
P^{*}\in\arg\max_{\overline{P},\boldsymbol{\mathbf{Q}}\in\mathcal{D}}\mathbb{E}_{\mathbf{H}}\left[u^{\textrm{MIMO}}\left(\mathbf{Q};\mathbf{H}\right)\right]\label{eq:threshold of the total power}
\end{equation}

$P^{*}$ can be found by comparison through Monte-Carlo simulation. One can imagine that finding the analytical decisional quantizer for EGT will be very difficult if the dimension of the system is huge. 
Thus we propose to use a neural network to mimic the real decisional quantizer.


\section{Simulation results}
\label{sec:Simulation results}
In this section, we present several simulations that illustrate the performance by using the proposed approach with neural-network. Here, we choose the 3-hidden-layer FNN with fully connected layers comprising 20
neurons each and using the logistic activation function defined
as $\mathrm{sig}(x)=\frac{1}{1+\exp(-x)}$. We use the \emph{Levenberg Marquardt Algorithm} in \cite{alsenl1963} to update the weight matrix. In this FNN model, $100000$ Monte-Carlo realizations will be divided into three phases: $70000$ realizations for the training phase, $15000$ for the validation phase  and  $15000$ realizations for the test phase. The structure of FNN is illustrated in Fig.~\ref{fig:neural network model}. Subsequently, simulation results will be illustrated in two different scenarios: the multiple band case and the multiple antenna case.

\begin{figure}[htbp]
\centerline{\includegraphics[width=1.0\linewidth,height=0.3\linewidth]{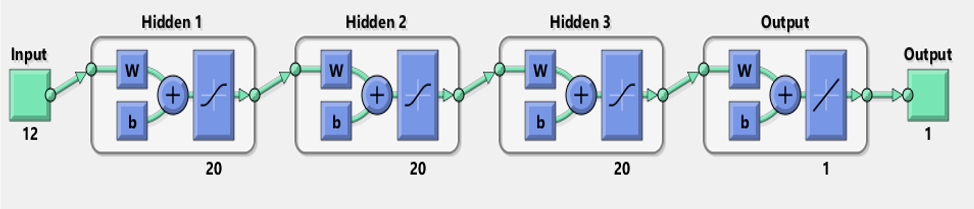}}
\caption{Feed-forward neural network model for MIMO system ($N_{t}=3$ and $N_{r}=2$). Number of neurons in input layer is $2N_tN_r$.}
\label{fig:neural network model}
\end{figure}

\textbf{Multi-band:} Firstly, we focus on the multi-band case. For EE defined in Eq. \ref{eq:energy efficiency definition}, The number of input neurons for Eq. \ref{eq:energy efficiency definition} is obviously the number of bands $N$.  Fig. \ref{fig:decision region 2 bands}  illustrates the decision regions for the following simulation configuration: there are two bands in the system ($N=2$), every band has only two choices to choose: $P_{min}=2\mathrm{mW}$, $P_{max}=3\mathrm{mW}$. The noisy level is set to be $\sigma^{2}=10\mathrm{mW}$ and the constant is assumed to be $c=1$. The channel gain $g_{i}$ in band $i$ is assumed to be exponentially distributed, i.e., its p.d.f. is $\phi\left(g_{i}\right)=\exp\left(-g_{i}\right)$.
There follows our intuitive explanation. Let us take the orange region $\left(P_{min},P_{max}\right)$ as an example. In this region, channel gain $g_{1}$ is smaller than $g_{2}$ which means transmission in band $1$ is less efficient than band $2$, therefore the transmitter chooses the policy $\left(P_{min},P_{max}\right)$. Same principle can be applied to the $3$ remaining regions.
\begin{figure}[htbp]
\begin{centering}
\includegraphics[scale=0.35]{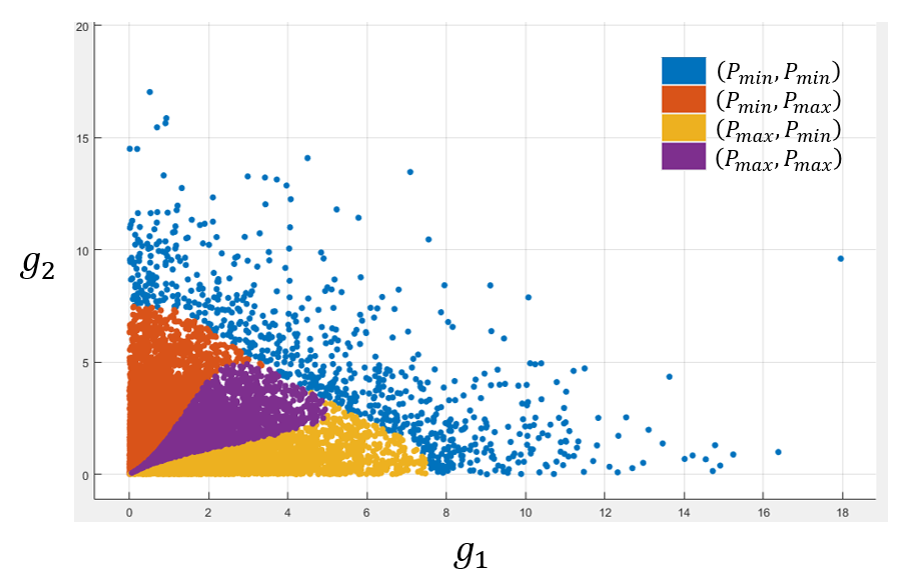}
\par\end{centering}
\caption{Decision regions of $2$-band EE problem. When one channel is dominant, it is better to transmit with higher power levels in that dominant channel. Otherwise, both transmitters choose the same transmit power.}
\label{fig:decision region 2 bands}
\end{figure}
Apart from the EE function, we study the sum-rate utility as follows: 
\begin{equation}
u^{\mathrm{SR}}(\textbf{p};\textbf{g}) = \sum_{i=1}^N  \log(1+\mathrm{SNR}_i)
\end{equation}
To compare the performance of DOQ by using
a decisional quantizer approach implemented through a FNN,
the relative optimality loss introduced by quantization is defined
as following:
\begin{equation}
\Delta u\left(\%\right)=\mathbb{E}_{\boldsymbol{g}}\left[\left|\frac{u^{*}\left(\boldsymbol{g}\right)-u^{\textrm{NN}}\left(\boldsymbol{g}\right)}{u^{*}\left(\boldsymbol{g}\right)}\right|\right]\times100\label{eq:relative optimality loss definition}
\end{equation}

where $u^{*}\left(\boldsymbol{g}\right)=\max_{\boldsymbol{p}}u\left(\boldsymbol{p};\boldsymbol{g}\right)$
and $u^{\textrm{NN}}\left(g\right)$ is the performance achieved by our learning approach.
Besides, to compare the influence of the compression between the system with different objectives, Define the compression rate $\gamma\left(\sigma\right)$ of a given
relative optimality loss $\sigma$  as
$\gamma\left(\sigma\right)\coloneqq\frac{M\left(1\%\right)}{M\left(\sigma\right)}$, where $M\left(\sigma\right)$ is the required number of decisions such
that the relative optimality loss $\sigma$ can be satisfied. Fig.~\ref{fig:compression_rate_optimality_loss}
illustrates the compression rate $\gamma$ in function of optimality loss
 for two bands in two cases. With the two different utilities, it can be seen that the compression rate increases as the optimality loss grows. For the energy efficiency problem, the compression
rate decreases slowly while the optimality loss decreases and the loss
is always greater than $1\%$. As for the sum-rate problem, the compression
rate declines rapidly while the optimality loss reduces and the optimality
loss is always less than $1\%$. It can be observed that it is easier to compress the $\textbf{g}$ for the  sum-rate problem than the energy efficiency in two-band scenario, i.e., the energy efficient function is more sensitive to the variable $g$. This can be explained by the fact that  the explicit optimal decision function of sum-rate, well known as the water-filling solution, is more concise than the solution of the energy efficiency problem, which is inversely proportional to $g$.

\begin{figure}[tbh]
\begin{centering}
\includegraphics[scale=0.4]{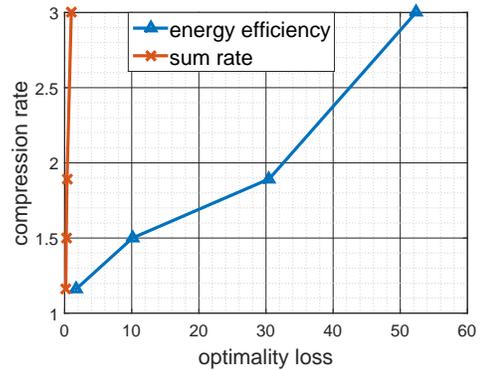}
\par\end{centering}
\caption{The compression rate as a function of the optimality loss for single
user 2-band scenario for energy efficiency and sum-rate capacity. Compressing the channel gain for sum-rate function is easier than compressing the channel gain for energy-efficiency function.}
\label{fig:compression_rate_optimality_loss}
\end{figure}


\textbf{MIMO:} The simulation results of the MIMO system considered in section \ref{subsec:Energy Efficiency in MIMO system} are presented in Fig. \ref{fig:Nt_4_Nr_1_average_utility} ($N_{t}=4$ and $N_{r}=1$ (MISO), $R_0=10^6 $ bits/s, $\sigma^2=5\mathrm{mW}$, $P_{0}=10\mathrm{mW}$ and $P_{max}=12\mathrm{mW}$.) and in Fig. \ref{fig:Nt_3_Nr_2_average_utility}, ($N_{t}=3$ and $N_{r}=2$ (MIMO), $R_0=10^6 $ bits/s, $\sigma^2=5\mathrm{mW}$,  $P_{0}=10\mathrm{mW}$  and $P_{\max}=10\mathrm{mW}$) (We use exhaustive search to find the optimal $P_{\max}$ by solving (\ref{eq:threshold of the total power})), respectively. The number of neurons in input layer for Eq. \ref{eq:single user MIMO EE with P0} is given by $2N_tN_r$ because the input  vector contains the real part and the  imaginary part of each entry of the transfer channel matrix. Given the same parameter samples, a $k$-means quantizer which aims at minimizing the mean square error between the original signal and the quantized signal,  is taken as the reference. All the realizations are divided into $k$ regions and each region is
assigned with the optimal decision in $\mathcal{D}_{k}$ found through
exhaustive research. It is worth noting that this $k$-means approach can be seen as a special case implementing Algo.~\ref{alg:Decisional-Quantization-Algorithm} by taking $u\left(\boldsymbol{x};\boldsymbol{g}\right)=-\left\Vert \boldsymbol{x}-\boldsymbol{g}\right\Vert ^{2}$.

In both two cases, the decisional quantizer outperforms
than the $k$-means quantizer. In MISO scenario, NN can achieve very
close performance to the optimal average utility in several decision
set ($\mathcal{D}_{k}$, $k=2,5,6,7$ and $k\geq9$) while the average
utility found through $k$-means quantizer is trite. In MIMO scenario,
the performance of NN is still better than $k$-means quantizer. The
utility loss introduced by the FNN is perhaps owing to the scarcity
of training. More complicated structure of NN should be used to improve
the training accuracy.

\begin{figure}[htbp]
\begin{centering}
\includegraphics[scale=0.40]{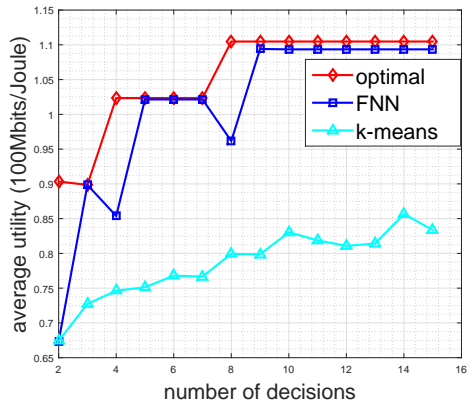}
\par\end{centering}
\caption{The average utility vs. number of decisions for $N_{t}=4$ and $N_{r}=1$ (MISO) , $\sigma^2=5\mathrm{mW}$, $P_{0}=10\mathrm{mW}$ and $P_{max}=12\mathrm{mW}$. FNN is better than k-means quantizer and close to theoretical optimal utility.}
\label{fig:Nt_4_Nr_1_average_utility}
\end{figure}

\begin{figure}[htbp]
\begin{centering}
\includegraphics[scale=0.40]{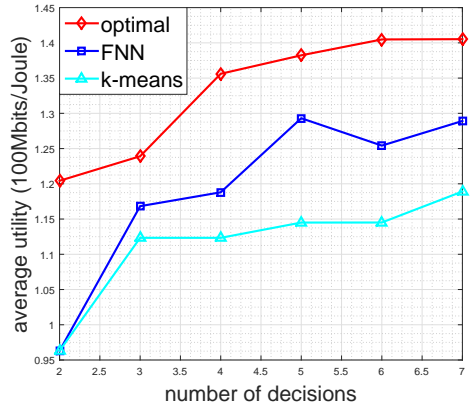}
\par\end{centering}
\caption{The average utility vs. number of decisions for $N_{t}=3$ and $N_{r}=2$ (MIMO) , $\sigma^2=5\mathrm{mW}$, $P_{0}=10\mathrm{mW}$ and $P_{max}=10\mathrm{mW}$. FNN is better than k-means quantizer and close to theoretical optimal utility. }
\label{fig:Nt_3_Nr_2_average_utility}
\end{figure}

\section{Conclusion}
\label{sec:Conclusion}
In this paper, we have introduced in a formal way the problem of DOC, and proposed solutions for DOQ. When applied to the problem of power allocation, it is seen that quantizing the channel gains very roughly only induces a very small optimality loss w.r.t. the case where the gains are perfectly known to the transmitter when the utility is the transmission rate. However, for energy-efficiency, channel gains need to quantized more accurately. Using a classical distortion-based quantization scheme (k-means quantization) for this is shown to lead to a quite significant performance loss (about $30\%$), showing the potential of our approach. To better assess the potential of the proposed approach, it should be generalized to decision-oriented source coding and decision-oriented channel coding. Also, it allows one to reconsider the overarching assumption made in resource allocation problem, that is the resource allocation policy is designed by assuming perfect knowledge of the parameters. Mathematically, a deep study should be developed to identifying the properties of the utility function which represents its sensitivity to being maximized under imperfect knowledge of its parameters.

%
%

\section*{Acknowledgement}
This work was funded by the RTE Chair and the PEPS YPSOC project funded by CNRS.

\end{document}